\documentclass[journal]{IEEEtai}

\let\labelindent\relax

\usepackage[colorlinks,urlcolor=blue,linkcolor=blue,citecolor=blue]{hyperref}

\usepackage{color}
\usepackage{array}
\usepackage{graphicx}
\usepackage{subfigure}
\usepackage{latexsym}
\usepackage{amssymb}
\usepackage{amsmath}
\usepackage{amsthm}
\usepackage{booktabs}
\usepackage{enumitem}
\usepackage{multirow} 
\usepackage{fdsymbol}
\usepackage{cite}

\newtheorem{theorem}{Theorem}

\newtheorem{definition}{Definition}
\newtheorem{assumption}{Assumption}
\begin{document}

\title{Causal Effect Estimation using identifiable Variational AutoEncoder with Latent Confounders and Post-Treatment Variables} 

\author{Anonymous}

 \author{Yang Xie, Ziqi Xu, Debo Cheng, Jiuyong Li, Lin Liu, Yinghao Zhang, and Zaiwen Feng
	 \thanks{This work was supported by the Chinese Central Universities (grant number: 2662023XXPY004) and the Australian Research Council (grant number: DP230101122). }
	 \thanks{Yang Xie, Yinghao Zhang, and Zaiwen Feng are with the College of Informatics, Huazhong Agricultural University, Wuhan, China (e-mail: Zaiwen.Feng@mail.hzau.edu.cn).}
	 \thanks{Ziqi Xu is with the Data 61, Commonwealth Scientific and Industrial Research Organisation (CSIRO), Australia (ziqi.xu@data61.csiro.au).}
	 \thanks{Debo Cheng, Jiuyong Li, and Lin Liu are with the STEM, University of South Australia, Adelaide, Australia (e-mail: debo.cheng@unisa.edu.au).}}
	
\markboth{Journal of IEEE Transactions on Artificial Intelligence, Vol. 00, No. 0, Month 2020}
{First A. Author \MakeLowercase{\textit{et al.}}: Bare Demo of IEEEtai.cls for IEEE Journals of IEEE Transactions on Artificial Intelligence}

\maketitle

\begin{abstract}
Estimating causal effects from observational data is challenging, especially in the presence of latent confounders. Much work has been done on addressing this challenge, but most of the existing research ignores the bias introduced by the post-treatment variables. In this paper, we propose a novel method of joint Variational AutoEncoder (VAE) and identifiable Variational AutoEncoder (iVAE) for learning the representations of latent confounders and latent post-treatment variables from their proxy variables, termed CPTiVAE, to achieve unbiased causal effect estimation from observational data. We further prove the identifiability in terms of the representation of latent post-treatment variables. Extensive experiments on synthetic and semi-synthetic datasets demonstrate that the CPTiVAE outperforms the state-of-the-art methods in the presence of latent confounders and post-treatment variables. We further apply CPTiVAE to a real-world dataset to show its potential application.
\end{abstract}

\begin{IEEEImpStatement}
\textcolor{black}{Deep learning models have shown great benefits in causal inference with latent variables, and increasingly, studies are utilizing the powerful representation and learning capabilities of neural networks for causal effect estimation from observational data. However, most existing studies inadvertently introduce post-treatment bias caused by post-treatment variables when making causal effect estimates from observational data. To address this, we use iVAE to learn the latent representations of post-treatment variables from their proxy variables for unbiased causal effect estimation. VAE has strong representation learning capabilities, and since post-treatment variables are affected by the treatment variable, it can effectively provide additional information during learning to ensure that the learned latent representation is identifiable. Finally, experiments on a large number of datasets show that the CPTiVAE model outperforms existing models.}
\end{IEEEImpStatement}


\section{Introduction}

\IEEEPARstart{C}{ausal} inference promotes science by discovering, understanding, and explaining natural phenomena, and has become increasingly prominent in a variety of fields, such as economics~\cite{imbens_rubin_2015, Donald_economics}, social sciences~\cite{johansson2018learning}, medicine~\cite{Medicne_Connors}, and computer science~\cite{Guo_2020}. Randomized controlled trials (RCTs) are considered the gold standard for quantifying causal effects. However, conducting RCTs is often not feasible due to ethical issues, high costs, and time constraints~\cite{NBERw22595}. Therefore, causal effect estimation using observational data is a common alternative to RCTs~\cite{cheng2022datadriven, Guo_2020}.

In causal inference, confounders are specific variables that simultaneously affect the treatment variable $T$ and the outcome variable $Y$. For example, in medicine, a patient's age can be a confounder when assessing the treatment effect of a drug for a particular disease. This is because younger patients are generally more likely to recover than older patients. Additionally, age can influence the choice of treatment, with different doses often prescribed to younger versus older patients for the same drug~\cite{yao2021survey}. Consequently, neglecting to account for age may lead to the erroneous conclusion that the drug is highly effective in treating the disease \textcolor{black}{when} it is fact not.

Substantial progress has been made in developing methods to address confounding bias when estimating causal effects from observational data. Such methods include back-door adjustment~\cite{pearl_2009}, confounding balance~\cite{johansson2018learning}, tree-based estimators~\cite{Athey2019} and matching methods~\cite{imbens_rubin_2015}. These algorithms have achieved significant success in estimating causal effects from observational data. However, the estimation may still be biased by potential post-treatment variables affected by the treatment itself~\cite{pearl_2009}. Acharya et al.~\cite{acharya2016explaining} reported that up to 80\% of observational studies were conditional on post-treatment variables. 

\begin{figure}[t]%
    \centering
    \subfigure[CEVAE]{
    \includegraphics[width=0.14\textwidth]{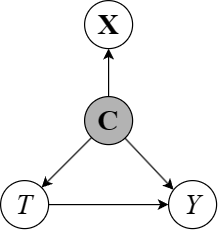}
    \label{CEVAE SCM}
    }
    \subfigure[TEDVAE]{
    \includegraphics[width=0.14\textwidth]{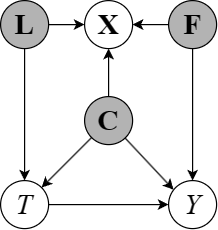}
    \label{TEDVAE SCM}
    }
    \subfigure[CPTiVAE]{
    \includegraphics[width=0.14\textwidth]{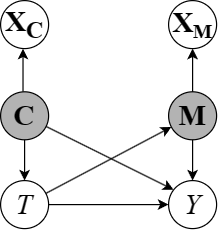}
    \label{CMiVAE SCM}
    }
    \caption{The causal graph assumed by CEVAE, TEDVAE, and CPTiVAE. In all figures, $T$ is the treatment, $Y$ is the outcome, and $\mathbf{C}$ are the latent confounders that affect both treatment and outcome. (a) The causal graph of CEVAE. $\mathbf{X}$ is the proxy for $\mathbf{C}$. (b) The causal graph of TEDVAE, $\mathbf{X}$ is the observed variables which may contain non-confounders and noisy proxy variables, $\mathbf{L}$ are factors that affect only the treatment, $\mathbf{F}$ are factors that affect only the outcome. (c) The causal graph for the proposed CPTiVAE. $\mathbf{M}$ are the latent post-treatment variables which are affected by the treatment variable $T$, $\mathbf{X_{C}}$ and $\mathbf{X_{M}}$ are proxies for $\mathbf{C}$ and $\mathbf{M}$, respectively.}
    \label{All SCM}
\end{figure}

Post-treatment variables bias the causal effect estimation from observational data. We use a simple example to illustrate this form of bias. An RCT is conducted to test whether a new drug can lower blood pressure in patients. In this scenario, the patients are randomly divided into two groups, one receiving the new drug (treatment group) and the other taking a placebo (control group). After receiving treatment, some patients become better and make lifestyle changes, such as diet and exercise. This lifestyle change may have affected their blood pressure, not just the effect of the drug itself. In this case, comparing the blood pressure levels of the two groups of patients before and after treatment may underestimate the actual effect of the new drug because post-treatment bias is not considered. Furthermore, these post-treatment variables are often unmeasured, which makes estimating causal effects more challenging. Accurate identification and adjustment of these variables are crucial for reliable causal inference.

To tackle the challenges mentioned above, we propose CPTiVAE, a novel identifiable Variational AutoEncoder (iVAE) method for unbiased causal effect estimation in the presence of latent confounders and post-treatment variables. In summary, our contributions are as follows:
\begin{itemize}
    \item We address an important and challenging problem in causal effect estimation using observational data with latent confounders and post-treatment variables.
    \item We propose a novel algorithm that jointly utilizes Variational AutoEncoder (VAE) and Identifiable VAE (iVAE) simultaneously to learn the representations of latent confounders and post-treatment variables from proxy variables, respectively. We further prove the identifiability in terms of the representation of latent post-treatment variables. To the best of our knowledge, this is the first work to tackle both confounding and post-treatment bias in causal effect estimation.
    \item We conduct extensive experiments on both synthetic and semi-synthetic datasets to evaluate the effectiveness of CPTiVAE. The results show that our algorithm outperforms existing methods. We also demonstrate its potential application on a real-world dataset.
\end{itemize}

\section{Related Work}
In recent years, numerous methods for estimating causal effects from observational data have been developed. A key challenge in performing causal inference from such data is tackling confounding bias, which becomes even more complex in the presence of latent post-treatment variables. In this section, we review some related works that address both confounding and post-treatment bias in causal inference.

Many methods assume that all confounders can be observed, i.e., the well-known unconfoundedness assumption~\cite{imbens_rubin_2015}. For example, Pocock et al.~\cite{Pocock2002SubgroupAC} proposed to control for all covariates by using an adjustment strategy, and the method in Li et al.~\cite{Li2018} balanced the covariates by re-weighting the propensity score. Some tree-based methods, such as CausalForestDML, design specific splitting criteria to estimate the causal effect~\cite{athey2016recursive, su2009subgroup}. Furthermore, many representation learning based methods have been developed for causal effect estimation from observational data, but they generally rely on the unconfoundedness assumption. Shalit et al.~\cite{shalit2017estimating} presented the method that shares outcome information between the treatment and control groups through representation learning while regularizing the distance of the representation distribution between the groups. Yoon et al.~\cite{yoon2018ganite} used a Generative Adversarial Nets (GAN) framework, GANITE  for estimating counterfactual outcomes based on a fitted generative adversarial network but only modelled the conditional expectation of the input outcomes. Kallus~\cite{kallus2020deepmatch} used weights defined in a representation space to balance treatment groups. However, these methods do not address the bias caused by post-treatment variables. 

Proxy variables are commonly used for causal inference~\cite{louizos2017causal,miao2018identifying, Wooldridge2009OnEF,xu2023disentangled,xu2023causal}. Miao et al.~\cite{miao2018identifying} proposed using proxies as general conditions for identifying causal effects, but this was not based on a data-driven method. Louizos et al.~\cite{louizos2017causal} introduced the CEVAE model (Figure~\ref{CEVAE SCM}), a neural network latent model that learns the representation of proxy variables for estimating individual and population causal effects. TEDVAE (Figure~\ref{TEDVAE SCM}) improves upon CEVAE by disentangling the representation learning for more accurate estimation~\cite{zhang2021treatment}. Additionally, some other Variational AutoEncoder (VAE) based estimators are related to our work. For example, $\beta$-VAE~\cite{higgins2016beta}, Factor VAE~\cite{kim2018disentangling}, HSIC-VAE~\cite{lopez2018information} and $\sigma$-VAE~\cite{rybkin2021simple} are both variants of VAE. However, the identifiability of the representations learned by these methods is not guaranteed.

Another line of work involves instrumental variables (IVs) methods, which are typically used to estimate causal effects when confounders are unmeasured in observational data~\cite{cheng2022datadriven}. Classical methods like the two-stage least squares model the relationship between treatments and outcomes using linear regression~\cite{kuang2020ivy}. Advances such as kernel IV~\cite{singh2019kernel} and Dual IV~\cite{muandet2020dual} employ more complex mappings for two-stage regression. A primary challenge for IV-based methods is identifying a valid IV. These IV-based methods are not directly related to our method CPTiVAE. 

In summary, traditional methods under the unconfoundedness assumption often fail to address latent confounders. Methods based on IVs can handle latent confounders, but obtaining valid instrumental variables is challenging. Methods based on representation learning and VAEs can effectively estimate causal effects from observational data but do not guarantee identifiability. Additionally, all these methods tend to overlook the presence of latent post-treatment variables, which may introduce bias in the estimation of causal effects from observational data.

\section{Preliminary}
Let $\mathcal{G}=(\mathbf{V},\mathbf{E})$ be a Directed Acyclic Graph (DAG), where $\mathbf{V}$ is the set of nodes and $\mathbf{E}$ is the set of edges between the nodes. In a causal DAG, a directed edge $V_{i} \to V_{j}$ signifies that variable of $V_{i}$ is a cause of variable of $V_{j}$ and $V_{j}$ is an effect variable of $V_{i}$. A path $\pi$ from $V_{i}$ to $V_{k}$ is a directed or causal path if all edges along it are directed towards $V_{k}$. If there is a directed path $\pi$ from $V_{i}$ to $V_{k}$, $V_{i}$ is known as an ancestor of $V_{k}$ and $V_{k}$ is a descendant of $V_{i}$. The sets of ancestors and descendants of a node $V$ are denoted as $An(V)$ and $De(V)$, respectively.

\begin{definition}[Markov property~\cite{pearl_2009}]
    Given a DAG $\mathcal{G}=(\mathbf{V}, \mathbf{E})$ and the joint probability distribution $P(\mathbf{V})$, $\mathcal{G}$ satisfies
the Markov property if for $\forall V_{i} \in \mathbf{V}$, $V_{i}$ is probabilistically independent of all of its non-descendants in $P(\mathbf{V})$, given the parent nodes of $V_{i}$.
\end{definition}

\begin{definition}[Faithfulness~\cite{spirtes2000causation}]
    Given a DAG $\mathcal{G}=(\mathbf{V}, \mathbf{E})$ and the joint probability distribution $P(\mathbf{V})$, $\mathcal{G}$ is faithful to a joint distribution $P(\mathbf{V})$ over $\mathbf{V}$ if and only if every independence present in $P(\mathbf{V})$ is entailed by $\mathcal{G}$ and satisfies the Markov property. A joint distribution $P(\mathbf{V})$ over $\mathbf{V}$ is faithful to $\mathcal{G}$ if and only if $\mathcal{G}$ is faithful to $P(\mathbf{V})$.
\end{definition}

When the Markov property and faithfulness are satisfied, we can use $d$-separation to infer the conditional independence between variables entailed in the DAG $\mathcal{G}$.

\begin{definition}[$d$-separation~\cite{pearl_2009}]
    A path $\pi$ in a DAG is said to be $d$-separated (or blocked) by a set of nodes $\mathbf{S}$ if and only if $(\romannumeral1)$ the path $\pi$ contains a chain $V_{i} \to V_{k} \to V_{j}$ or a fork \textcolor{black}{$V_{i} \gets V_{k} \to V_{j}$} such that the middle node $V_{k}$ is in $\mathbf{S}$, or $(\romannumeral2)$ the path $\pi$ contains an inverted fork (or collider) $V_{i} \to V_{k} \gets V_{j}$ such that $V_{k}$ is not in $\mathbf{S}$ and no descendant of $V_{k}$ is in $\mathbf{S}$.
\end{definition}

In a DAG $\mathcal{G}$, a set $\mathbf{S}$ is said to $d$-separate $V_{i}$ from $V_{j} (V_{i} \Vbar V_{j} | \mathbf{S})$ if and only if $\mathbf{S}$ blocks every path between $V_{i}$ to $V_{j}$. Otherwise, they are said to be $d$-connected by $\mathbf{S}$, denoted as $V_{i} \nVbar V_{j} | \mathbf{S}$.

\begin{definition}[Back-door criterion~\cite{pearl_2009}]
In a DAG $\mathcal{G}=(\mathbf{V}, \mathbf{E})$, for the pair of variables $(T,Y) \in \mathbf{V}$, a set of variables $\mathbf{Z} \subseteq \mathbf{V} \setminus \left \{T,Y \right \} $ satisfy the back-door criterion in the given DAG $\mathcal{G}$ if (1) $\mathbf{Z}$ does not contain a descendant node of $T$; and (2) $\mathbf{Z}$ blocks every path between $T$ and $Y$ that contains \textcolor{black}{an arrow} to $T$. If $\mathbf{Z}$ satisfies the back-door criterion relative to $(T,Y)$ in $\mathcal{G}$, we have $ATE(T,Y)=\mathbb{E}[Y|T=1,\mathbf{Z}=z]-\mathbb{E}[Y|T=0,\mathbf{Z}=z]$.
\end{definition}

\textcolor{black}{In this work, we define $\mathbf{X_{C}}$ to be the set of pre-treatment variables, which are not affected by the treatment variable. We define $\mathbf{X_{M}}$ to be the set of post-treatment variables, which are causally affected by the treatment variable. For example, when estimating the average causal effect between education and income, basic information about the individual such as age, gender, and nationality can be considered as pre-treatment variables. In this case, the individual's abilities will improve after receiving education, such as learning some job skills through education, which can be considered as the post-treatment variable. Unfortunately, most studies have neglected post-treatment variables, treating them as pre-treatment variables, which introduces post-treatment bias into causal effect estimation.}

\section{The Proposed CPTiVAE Method}
\subsection{Problem Setting}
Our proposed causal model is shown in Figure~\ref{CMiVAE SCM}. Let $\mathbf{C}$ be the latent confounders and $\mathbf{M}$ be the latent post-treatment variables, $\mathbf{X_{C}} \in \mathbb{R} ^{a}$ and $\mathbf{X_{M}} \in \mathbb{R} ^{b}$ are the observed proxy variables for them respectively. \textcolor{black}{In this paper, confounders are variables that simultaneously affect the
treatment $T$ and the outcome $Y$. Post-treatment are variables that are affected by the treatment $T$ and impact the outcome $Y$. Confounders and post-treatment variables are latent if they are not directly observed in the data. Referring to Figure~\ref{CMiVAE SCM}, a post-treatment variable is a cause of $Y$, similar to a confounder, but the difference between them is that a confounder is a cause of treatment (i.e., $\mathbf{C} \to T$) whereas a post-treatment variable is affected by the treatment (i.e., $T \to \mathbf{M}$).} For simplicity, we denote the whole \textcolor{black}{proxy} variables as the concatenation $\mathbf{X} = \left [ \mathbf{X_{C}}, \mathbf{X_{M}} \right ] \in \mathbb{R} ^{n}$, where $n=a+b$. 

Let $T$ denote a binary treatment, where $T=0$ indicates a unit receives no treatment (control) and $T=1$ indicates a unit receives the treatment (treated). Let $Y=T\times Y(1) + \left ( 1-T \right ) \times Y(0)$ denote the observed outcomes, where $Y(0)$ and $Y(1)$ denote the potential outcomes in the control group and treatment group, respectively. 

Note that, for an individual, we can only observe one of $Y(0)$ and $Y(1)$, and those unobserved potential outcomes are called counterfactual outcomes~\cite{imbens_rubin_2015,rosenbaum1983central,rubin1979using}. 

The individual treatment effect (ITE) for $i$ is defined as:
\begin{equation}
    \text{ITE}_{i}=Y_{i}(T=1)-Y_{i}(T=0)
    \label{ITE}
\end{equation}
where $Y_{i}(T=1)$ and $Y_{i}(T=0)$ are potential outcomes for individual $i$ in the treatment and control groups, respectively.

The average treatment effect (ATE) of $T$ on $Y$ at the population level is defined as:
\begin{equation}
    \text{ATE}(T,Y)=\mathbb{E}[Y_{i}(T=1)-Y_{i}(T=0)]
    \label{ATE}
\end{equation}
where $\mathbb{E}$ indicates the expectation function. 

We use the ``do'' operation introduced by~\cite{pearl_2009}, denoted as $do(\cdot)$. ATE can be expressed as follows:
\begin{equation}
    \text{ATE}(T,Y)=\mathbb{E}[Y|do(T=1)]-\mathbb{E}[Y|do(T=0)]
    \label{ATE_do}
\end{equation}

In general, ITE cannot be obtained directly from the observed data according to Eq.~\ref{ITE}, since only one potential outcome can be observed for an individual. There are many data-driven methods for estimating ATE from observational data, but they introduce confounding bias. Covariate adjustment~\cite{de2011covariate,perkovi2018complete} and confounding balance~\cite{shalit2017estimating} are common methods for estimating the causal effect of $T$ on $Y$ unbiasedly from observational data. To estimate the causal effect of $T$ on $Y$, the back-door criterion~\cite{pearl_2009} is commonly employed to identify an adjustment set $\mathbf{Z} \subseteq \mathbf{X}$ within a given DAG $\mathcal{G}$. Then, the set $\mathbf{Z}$ is used to adjust for confounding bias in causal effect estimations of $T$ on $Y$.

For some groups with the same features, some researchers generally use conditional average treatment effect (CATE) to approximate ITE~\cite{shalit2017estimating}, defined as:
\begin{equation}
    \begin{split}
        \text{CATE}(T,Y|\mathbf{X}=x)&=\mathbb{E}[Y|do(T=1),\mathbf{X}=x]\\&~~~~~~-\mathbb{E}[Y|do(T=0),\mathbf{X}=x]
    \end{split}
    \label{CATE}
\end{equation}

In this work, we consider a problem with the presence of latent confounders $\mathbf{C}$ and post-treatment variables $\mathbf{M}$. Randomization is done to establish similarity between the treatment and control groups to ensure that the treatment and control groups are similar in all but the treatment variables, which helps to eliminate latent confounders. The post-treatment variables may disrupt the balance between the treatment and control group, which eliminates the advantages of randomization~\cite{Montgomery2018HowCO}. Consequently, $ATE(T, Y)$ and $CATE(T, Y|\mathbf{X}=x)$, cannot be directly calculated using Eqs.~\ref{ATE} and~\ref{CATE}. In this study, we propose the CPTiVAE algorithm to tackle the confounding bias and post-treatment bias simultaneously. 

\subsection{The Proposed CPTiVAE Algorithm}
CPTiVAE aims to learn the latent confounder $\mathbf{C}$ from the proxy variables $\mathbf{X_{C}}$, and the latent post-treatment $\mathbf{M}$ from the proxy variables $\mathbf{X_{M}}$ (see Figure~\ref{CMiVAE SCM}). Subsequently, the representations $\left \{\mathbf{C},\mathbf{M}\right\}$ are used to address confounding and post-treatment biases. 

Our CPTiVAE algorithm consists of two main components, VAE and iVAE~\cite{Khemakhem2019VariationalAA}. Specifically, we use VAE to learn the representation $\mathbf{C}$ from the proxy variables $\mathbf{X_{C}}$ for addressing the confounding bias as done in the works~\cite{louizos2017causal,zhang2021treatment}. \textcolor{black}{When the representation of confounders and the representation of the potential post-treatment variables are correct, the causal effect estimation is unbiased.}

To guarantee the identifiability in terms of the learned representation $\mathbf{M}$ by  CPTiVAE, we consider the treatment $T$ as an additional observed variable and use iVAE to learn the representation $\mathbf{M}$ from the observed proxy variables $\mathbf{X_{M}}$. 

To guarantee the identifiability of $\mathbf{C}$, we take $T$ as additionally observed variables to approximate the prior $p(\mathbf{M}|T)$~\cite{Khemakhem2019VariationalAA}. Following the standard VAE~\cite{kingma2013auto}, we use Gaussian distribution to initialize the prior distribution of $p(\mathbf{C})$ and also assume the prior $p(\mathbf{M}|T)$ follows the Gaussian distribution.

In the inference model of CPTiVAE, the variational approximations of the posteriors are defined as:
\begin{equation}
    \begin{aligned}
        q_{\theta_{C}}(\mathbf{C}|\mathbf{X_{C}})&=\prod_{i=1}^{D_{\mathbf{C}}}\mathcal{N}(\mu=\hat{\mu}_{C_{i}},\sigma^2=\hat{\sigma}^2_{C_{i}})\\
        q_{\theta_{M}}(\mathbf{M}|T,\mathbf{X_{M}})&=\prod_{i=1}^{D_{\mathbf{M}}}\mathcal{N}(\mu=\hat{\mu}_{M_{i}},\sigma^2=\hat{\sigma}^2_{M_{i}})
    \end{aligned}
    \label{inference C and M}
\end{equation}
where $\hat{\mu}_{C_{i}}$, $\hat{\mu}_{M_{i}}$ and $\hat{\sigma}^2_{C_{i}}$, $\hat{\sigma}^2_{M_{i}}$ are the estimated means and variances of $C_{i}$ and $M_{i}$, respectively.

The genereative model for $\mathbf{X_{M}}$ and $\mathbf{X_{C}}$ are difined as:
\begin{equation}
    \begin{aligned}
        p_{\varphi_{X_{C}}}&(\mathbf{X_{C}}|\mathbf{C})=\\&~~~~~~\prod_{i=1}^{D_{\mathbf{X_{C}}}}P(X_{C_{i}}|\mathcal{N}(\mu=\hat{\mu}_{\mathbf{X_{C}}_{i}},\sigma^2=\hat{\sigma}^2_{\mathbf{X_{C}}_{i}}))\\
        p_{\varphi_{X_{M}}}&(\mathbf{X_{M}}|\mathbf{M})=\\&~~~~~~\prod_{i=1}^{D_{\mathbf{X_{M}}}}P(X_{M_{i}}|\mathcal{N}(\mu=\hat{\mu}_{\mathbf{X_{M}}_{i}},\sigma^2=\hat{\sigma}^2_{\mathbf{X_{M}}_{i}}))\\
    \end{aligned}
    \label{Generative function X}
\end{equation}
where $\hat{\mu}^2_{\mathbf{X_{C}}_{i}} = g_1(\mathbf{C}_i)$; $\hat{\sigma}^2_{\mathbf{X_{C}}_{i}} = g_1(\mathbf{C}_i)$; $\hat{\mu}^2_{\mathbf{X_{M}}_{i}} = g_2(\mathbf{M}_i)$; $\hat{\sigma}^2_{\mathbf{X_{M}}_{i}} = g_2(\mathbf{M}_i)$, $g_1(\cdot)$ and $g_2(\cdot)$ are the neural network parameterised by their own parameters. $D_{\mathbf{X_{C}}}$ and $D_{\mathbf{X_{M}}}$ indicate the dimensions of $\mathbf{X_{C}}$ and $\mathbf{X_{M}}$, respectively. Note that $P(X_{C_{i}}|\mathbf{C})$ and $P(X_{M_{i}}|\mathbf{M})$ are the distributions on the $i$-th variable.

The generative model for $T$ is defined as:
\begin{equation}
    \begin{aligned}
        &p_{\varphi_{T}}(T|\mathbf{C})=Bern(\sigma(h_{1}(\mathbf{C}))) \\
    \end{aligned}
    \label{Generative function T}
\end{equation}
where $Bern(\cdot)$ is the Bernoulli function, $h_{1}(\cdot)$ is a neural network, $\sigma(\cdot)$ is the logistic function.

Specifically, the generative models for $Y$ vary depending on the types of the attribute values. For continuous $Y$, we parameterize it as a Gaussian distribution with its mean and variance given by the neural network. We define the control groups and treatment groups as $P(Y|T=0,\mathbf{C},\mathbf{M})$ and $P(Y|T=1,\mathbf{C},\mathbf{M})$, respectively. Therefore, the generative model for $Y$ is defined as:
\begin{equation}
    \begin{aligned}
        &p_{\varphi_{Y}}(Y|T,\mathbf{C},\mathbf{M})=\mathcal{N}(\mu =\hat{\mu}_{Y}, \sigma^2=\hat{\sigma}_{Y}^2)\\
        &\hat{\mu}_{Y}=T \cdot h_{2}(\mathbf{C},\mathbf{M})+(1-T) \cdot h_{3}(\mathbf{C},\mathbf{M})\\
        &\hat{\sigma}_{Y}^2=T \cdot h_{4}(\mathbf{C},\mathbf{M})+(1-T) \cdot h_{5}(\mathbf{C},\mathbf{M})
    \end{aligned}
    \label{continuous Y}
\end{equation}
where $h_{2}(\cdot)$, $h_{3}(\cdot)$, $h_{4}(\cdot)$ and $h_{5}(\cdot)$ are the functions parameterized by neural networks, $\hat{\mu}_{Y}$ and $\hat{\sigma}_{Y}^2$ are the means and variances of the Gaussian distributions parametrized by neural networks.

For binary $Y$, it can be similarly parameterized with a Bernoulli distribution, defined as:
\begin{equation}
    p_{\varphi_{Y}}(Y|T,\mathbf{C},\mathbf{M})=Bern(\sigma(h_{6}(T, \mathbf{C},\mathbf{M})))
    \label{binary Y}
\end{equation}
where $h_{6}(\cdot)$ is a neural network.

We combine the inference model and the generative model into a single objective, and the variational evidence lower bound (ELBO) of this model is defined as:
\begin{equation}
    \begin{split}
        \mathcal{L}_{\text{ELBO}}&=\mathbb{E}_{q_{\theta_{C}}q_{\theta_{M}}}[\log p_{\varphi_{X}}(\mathbf{X}|\mathbf{C}, \mathbf{M})]\\&~~~~~~-D_{KL}[q_{\theta_{C}}(\mathbf{C}|\mathbf{X_{C}})||p(\mathbf{C})]\\&~~~~~~-D_{KL}[q_{\theta_{M}}(\mathbf{M}|T,\mathbf{X_{M}})||p(\mathbf{M}|T)]
    \end{split}
    \label{ELBO}
\end{equation}
where $D_{KL}[\cdot || \cdot]$ is a KL divergence term.

To ensure that $T$ is accurately predicted by $\mathbf{C}$, and $Y$ is correctly predicted by $T$, $\mathbf{C}$, and $\mathbf{M}$, we incorporate two auxiliary predictors into the variational ELBO. Consequently, the objective of CPTiVAE can be defined as:
\begin{equation}
    \begin{split}
        \mathcal{L}_{\text{CPTiVAE}}&=-\mathcal{L}_{\text{ELBO}}+\alpha \mathbb{E}_{q_{\theta_{C}}}[\log q(T|\mathbf{C})]\\&~~~~~~+\beta \mathbb{E}_{q_{\theta_{C}}q_{\theta_{M}}}[\log q(Y|T,\mathbf{C}, \mathbf{M})]
    \end{split}
    \label{loss fuction}
\end{equation}
where $\alpha$ and $\beta$ are the weights for balancing the two auxiliary predictors.

\subsection{Identifiability in terms of the Learned Representation by CPTiVAE}
In the preceding subsections, we discussed the generative process and optimization objective of the CPTiVAE method. Now, we provide the main theorem of this paper: the identifiability in terms of the learned representation $\mathbf{M}$ by CPTiVAE.

Let $\theta=(f,H,\lambda)$ be the parameters on $\Theta$ of the following conditional generative model:
\begin{equation}
    p_{\theta}(\mathbf{X_{M}}, \mathbf{M}|T)=p_{f}(\mathbf{X_{M}}|\mathbf{M})p_{H,\lambda}(\mathbf{M}|T)
    \label{eq12} 
\end{equation}
where $p_{f}(\mathbf{X_{M}}|\mathbf{M})$ is defined as:
\begin{equation}
    p_{f}(\mathbf{X_{M}}|\mathbf{M})=p_{\varepsilon}(\mathbf{X_{M}-f(\mathbf{M})})
    \label{p_f(x|z)}
\end{equation}
in which the value of $\mathbf{X_{M}}$ is decomposed as $\mathbf{X_{M}}=f(\mathbf{M})+\varepsilon$, where $\varepsilon$ is an independent noise variable with probability density function $p_{\varepsilon}(\varepsilon)$, i.e. $\varepsilon$ is independent of $\mathbf{M}$ or $f$.

We assume the conditional distribution $p_{H,\lambda}(\mathbf{M}|T)$ is conditional factorial with an exponential family distribution, which shows the relation between latent variables $\mathbf{M}$ and observed variables $T$~\cite{Khemakhem2019VariationalAA}.
\begin{assumption}
    The conditioning on $T$ is through an arbitrary function $\lambda(T)$ (such as a look-up table or neural network) that outputs the individual exponential family parameters $\lambda_{i,j}$. The probability density function is thus given by:  
    \begin{equation}
    \begin{split}
            &p_{H,\lambda}(\mathbf{M}|T)=\\&~~~~~~~~~\prod_{i}\frac{Q_{i}(M_{i})}{Z_{i}(T)} \exp[\sum_{j=1}^{k}H_{i,j}(M_{i})\lambda_{i,j}(T)]
        \label{p_exp}
    \end{split}
    \end{equation}
    where $Q_{i}$ is the base measure, $Z_{i}(T)$ is the normalizing constant, $\mathbf{H}_{i}=(H_{i,1},...,H_{i,k})$ are the sufficient statistics, $\mathbf{\lambda}_{i}(T)=(\lambda_{i,1}(T),...,\lambda_{i,k}(T))$ are the $T$ dependent parameters, and $k$, the dimension of each sufficient statistic, is fixed. 
\end{assumption}

To prove the identifiability of CPTiVAE, we introduce the following definitions~\cite{Khemakhem2019VariationalAA}:
\begin{definition}[Identifiability classes]
    Let $\sim$ be an equivalence relation on $\Theta$.  The Eq.~\ref{eq12} is identifiable up to $\sim$ (or $\sim$-identifiable) if
    \begin{equation}
        p_{\theta}(x)=p_{\tilde{\theta}}(x) \Longrightarrow \tilde{\theta} \sim \theta
        \label{Identifiability classes function}
    \end{equation}
    The elements of the quotient space $\Theta / \sim$ are called the identifiability classes.
    \label{Identifiability classes}
\end{definition}

According to Definition.~\ref{Identifiability classes}, we now define the equivalence relation on the set of parameters $\Theta$~\cite{Cai2022LongtermCE}.
\begin{definition}
     Let $\sim_{A}$ be the binary relation on $\Theta$ defined as follows:
     \begin{equation}
         \begin{split}
             &(f,H,\lambda) \sim_{A} (\hat{f},\hat{H},\hat{\lambda}) \Longleftrightarrow \\
             &\exists A,c | H(f^{-1}(x))=A\hat{H}(\hat{f}^{-1}(x))+c ,\forall x \in \mathcal{X}
         \end{split}
         \label{Ac}
     \end{equation}
     where $A$ is a invertible matrix and $c$ is a vector, and $\mathcal{X}$ is the domain of $x$.
\end{definition}

We can conclude the identifiability of CPTiVAE as follows:
\begin{theorem}
    Assume that the data we observed are sampled from a generative model defined according to Eqs.~\ref{eq12}-\ref{p_exp}, with parameters $(f, H,\lambda)$. Suppose the following conditions hold:
    \begin{itemize}
        \item[({\romannumeral1})]The mixture function $f$ in Eq.~\ref{p_f(x|z)} is injective.
        \item[({\romannumeral2})]The set $\left \{x \in \mathcal{X}|\phi_{\varepsilon}(x)=0 \right \}$ has measure zero, where $\phi_{\varepsilon}$ is the characteristic function of the density $p_{f}$ in Eq.~\ref{p_f(x|z)}.
        \item[({\romannumeral3})]The sufficient statistics $H_{i,j}$ are differentiable almost everywhere, and $(H_{i,j})_{1 \le j \le k}$ are linearly independent on any subset of $\mathcal{X}$ of measure greater than zero.
        \item[({\romannumeral4})]There exist $nk+1$ distinct points $T_{0},...,T_{nk}$ such that the matrix
        \begin{equation}
            L=(\lambda(T_{1})-\lambda(T_{0}),...,\lambda(T_{nk})-\lambda(T_{0}))
        \end{equation}
        of size $nk\times nk$ is invertible.
    \end{itemize}
    Then the parameters $(f,H,\lambda)$ are $\sim_{A}$-identifiable. 
    \label{identifiability theorem}
\end{theorem}

\begin{proof}
The proof of this Theorem is based on mild assumptions. Suppose we have two sets of parameters $\theta=(f,H,\lambda)$ and $\hat{\theta}=(\hat{f},\hat{H},\hat{\lambda})$ such that $p_{\theta}(\mathbf{M}|T)=p_{\hat{\theta}}(\mathbf{M}|T)$ for all pairs $(\mathbf{M},T)$. Then:    
\begin{align}
    &\int_{\mathcal{M}}p_{f}(\mathbf{X_{M}}|\mathbf{M})p_{H,\lambda}(\mathbf{M}|T)\mathrm{d}\mathbf{M} \notag \\
    =&\int_{\mathcal{M}}p_{\hat{f}}(\mathbf{X_{M}}|\mathbf{M})p_{\hat{H},\hat{\lambda}}(\mathbf{M}|T)\mathrm{d}\mathbf{M} \label{first_Eq} \\
    \Longrightarrow &\int_{\mathcal{X}}p_{H,\lambda}(f^{-1}(\overline{\mathbf{X}}_{\mathbf{M}})|T)volJ_{f^{-1}}(\overline{\mathbf{X}}_{\mathbf{M}}) \notag \\
    &\qquad \qquad \qquad \qquad p_{\varepsilon}(\overline{\mathbf{X}}_{\mathbf{M}}|f^{-1}(\overline{\mathbf{X}}_{\mathbf{M}}))\mathrm{d}\overline{\mathbf{X}}_{\mathbf{M}} \notag \\
    =&\int_{\mathcal{X}}p_{\hat{H},\hat{\lambda}}(\hat{f}^{-1}(\overline{\mathbf{X}}_{\mathbf{M}})|T)volJ_{\hat{f}^{-1}}(\overline{\mathbf{X}}_{\mathbf{M}}) \notag \\
    &\qquad \qquad \qquad \qquad p_{\hat{\varepsilon}}(\overline{\mathbf{X}}_{\mathbf{M}}|\hat{f}^{-1}(\overline{\mathbf{X}}_{\mathbf{M}}))\mathrm{d}\overline{\mathbf{X}}_{\mathbf{M}} \label{dX_M} \\
    \Longrightarrow &\int_{\mathbb{R}^{d}}\hat{p}_{H,\lambda,f,T}(\overline{\mathbf{X}}_{\mathbf{M}})p_{\varepsilon}(\overline{\mathbf{X}}_{\mathbf{M}}|f^{-1}(\overline{\mathbf{X}}_{\mathbf{M}}))\mathrm{d}\overline{\mathbf{X}}_{\mathbf{M}} \notag \\
    =&\int_{\mathbb{R}^{d}}\hat{p}_{\hat{H},\hat{\lambda},\hat{f},T}(\overline{\mathbf{X}}_{\mathbf{M}})p_{\varepsilon}(\overline{\mathbf{X}}_{\mathbf{M}}|f^{-1}(\overline{\mathbf{X}}_{\mathbf{M}}))\mathrm{d}\overline{\mathbf{X}}_{\mathbf{M}}  \label{Rd} \\
    \Longrightarrow &(\hat{p}_{H,\lambda,f,T}*p_{\varepsilon})(\overline{\mathbf{X}}_{\mathbf{M}})=(\hat{p}_{\hat{H},\hat{\lambda},\hat{f},T}*p_{\varepsilon})(\overline{\mathbf{X}}_{\mathbf{M}}) \label{*} \\
    \Longrightarrow &F[\hat{p}_{H,\lambda,f,T}](\omega)\varphi_{\varepsilon}(\omega)=F[\hat{p}_{\hat{H},\hat{\lambda},\hat{f},T}](\omega)\varphi_{\varepsilon}(\omega)  \label{Fphi} \\
    \Longrightarrow &F[\hat{p}_{H,\lambda,f,T}](\omega)=F[\hat{p}_{\hat{H},\hat{\lambda},\hat{f},T}](\omega) \label{Fourier}\\
    \Longrightarrow &\hat{p}_{H,\lambda,f,T}(\mathbf{X_{M}})=\hat{p}_{\hat{H},\hat{\lambda},\hat{f},T}(\mathbf{X_{M}}) \label{last_Eq}
\end{align}

In Eq.~\ref{dX_M}, $vol A$ is the volume of a matrix $A$. When $A$ is full column rank, $volA=\sqrt{A^{T}A}$, and when $A$ is invertible, $volA=|detA|$. $J$ denotes the Jacobian. We made the change of variable $\overline{\mathbf{X}}_{M}=f(\mathbf{M})$ on the left hand side, and $\overline{\mathbf{X}}_{M}=\hat{f}(\mathbf{M})$ on the right hand side.

We introduce the following equation:
\begin{align}
    &p_{H,\lambda}(f^{-1}(\overline{\mathbf{X}}_{\mathbf{M}})|T)volJ_{f^{-1}}(\overline{\mathbf{X}}_{\mathbf{M}}) \notag \\
   &=~~~p_{\hat{H},\hat{\lambda}}(\hat{f}^{-1}(\overline{\mathbf{X}}_{\mathbf{M}})|T)volJ_{\hat{f}^{-1}}(\overline{\mathbf{X}}_{\mathbf{M}})
   \label{p_vol}
\end{align}

By taking the logarithm on both side of Eq.~\ref{p_vol} and replacing $p_{H,\lambda}$ by its expression from Eq. 13, we have:
\begin{align}
    \log volJ_{f^{-1}}(\overline{\mathbf{X}}_{\mathbf{M}})+&\sum_{i=1}^{n}(\log Q_{i}(f_{i}^{-1}(\overline{\mathbf{X}}_{\mathbf{M}})) \notag \\
    -\log Z_{i}(T)+&\sum_{j=1}^{k}H_{i,j}(f_{i}^{-1}(\overline{\mathbf{X}}_{\mathbf{M}}))\lambda_{i,j}(T)) \notag \\
    = \log volJ_{\hat{f}^{-1}}(\overline{\mathbf{X}}_{\mathbf{M}})+&\sum_{i=1}^{n}(\log \hat{Q}_{i}(\hat{f}_{i}^{-1}(\overline{\mathbf{X}}_{\mathbf{M}})) \notag \\
    -\log \hat{Z}_{i}(T)+&\sum_{j=1}^{k}\hat{H}_{i,j}(\hat{f}_{i}^{-1}(\overline{\mathbf{X}}_{\mathbf{M}}))\hat{\lambda}_{i,j}(T))
    \label{log_vol}
\end{align}
Let $nk+1$ distinct points $T_{0},...,T_{nk}$ be the point provided by the assumption $(\romannumeral4)$ of the Theorem, and define $\overline{\lambda}(T)=\lambda(T)-\lambda(T_{0})$. We plug each of those $T_{l}$ in Eq.~\ref{log_vol} to obtain $nk+1$ equations, and we subtract the first equation from the remaining $nk$ equations to obtain for $l=1,...,nk$:
\begin{align}
    &\left \langle \mathbf{H}(f^{-1}(\overline{\mathbf{X}}_{\mathbf{M}})), \overline{\lambda}(T_{l}) \right \rangle + \sum_{i}\log \frac{Z_{i}(T_{0})}{Z_{i}(T_{l})}  \notag \\
    &~~~=\left \langle \hat{\mathbf{H}}(\hat{f}^{-1}(\overline{\mathbf{X}}_{\mathbf{M}})), \overline{\hat{\lambda}}(T_{l}) \right \rangle + \sum_{i}\log \frac{\hat{Z}_{i}(T_{0})}{\hat{Z}_{i}(T_{l})}
    \label{<T>}
\end{align}
Let $L$ be the matrix defined in assumtion $(\romannumeral4)$ and $\hat{L}$ similarly defined for $\hat{\lambda}$ ($\hat{L}$ is not necesssarily invertible). Define $b_{l}=\sum_{i}\log \frac{\hat{Z}_{i}(T_{0})Z_{i}(T_{l})}{Z_{i}(T_{0})\hat{Z}_{i}(T_{l})}$ and $b$ the vector of all $b_{l}$ for $l=1,...,nk$. Then Eq.~\ref{<T>} can be rewritten in the matrix form:
\begin{equation}
    L^{T}\mathbf{H}(f^{-1}(\overline{\mathbf{X}}_{\mathbf{M}}))=\hat{L}^{T}\hat{\mathbf{H}}(\hat{f}^{-1}(\overline{\mathbf{X}}_{\mathbf{M}}))+b
    \label{LH}
\end{equation}

Then we multiply both sides of the above equations by $L^{-T}$ to find:
\begin{equation}
    \mathbf{H}(f^{-1}(\overline{\mathbf{X}}_{\mathbf{M}}))=A\hat{\mathbf{H}}(\hat{f}^{-1}(\overline{\mathbf{X}}_{\mathbf{M}}))+c
    \label{H}
\end{equation}
where $A=L^{-T}\hat{L}$ and $c=L^{-T}b$.

Now we show that $A$ is invertible. By definition of $H$ and according to the assumption $(\romannumeral3)$ of the Theorem, its Jacobian exists and is an $nk \times n$ matrix of rank $n$. This implies that the Jacobian of $\hat{H} \circ \hat{f}^{-1}$ exists and is of rank $n$ and so is $A$. There are two cases:
(1) If $k=1$, $A$ is $n \times n$ matrix of rank $n$ and invertible; (2) If $k>1$, we define $\overline{\mathbf{X}}_{\mathbf{M}}=f^{-1}(\overline{\mathbf{X}}_{\mathbf{M}})$ and $\mathbf{H}_{i}(\overline{\mathbf{X}}_{\mathbf{M}i})=(H_{i,1}(\overline{\mathbf{X}}_{\mathbf{M}i}), ..., H_{i,k}(\overline{\mathbf{X}}_{\mathbf{M}i}))$. For each $i \in [1, ..., n]$ there exsit $k$ points $\overline{\mathbf{X}}_{\mathbf{M}i}^{1}, ..., \overline{\mathbf{X}}_{\mathbf{M}i}^{k}$ such that $(\mathbf{H}_{i}^{\prime}(\overline{\mathbf{X}}_{\mathbf{M}i}^{1}), ..., \mathbf{H}_{i}^{\prime}(\overline{\mathbf{X}}_{\mathbf{M}i}^{k}))$ are linearly independent. We suppose that for any choice of such $k$ points, the family $(\mathbf{H}_{i}^{\prime}(\overline{\mathbf{X}}_{\mathbf{M}i}^{1}), ..., \mathbf{H}_{i}^{\prime}(\overline{\mathbf{X}}_{\mathbf{M}i}^{k}))$ is never linearly independent. That means that $\mathbf{H}_{i}^{\prime}(\mathbb{R})$ is included in a subspace of $\mathbb{R}^{k}$ of the dimension at most $k-1$. Let $g$ be a non-zero vector that is orthigonal to $\mathbf{H}_{i}^{\prime}(\mathbb{R})$. Then for all $X_{M} \in \mathbb{R}$, we have $<\mathbf{H}_{i}^{\prime}(\mathbb{R}), g>=0$. By integrating we find that $<\mathbf{H}_{i}^{\prime}(\mathbb{R}), g>=const$. Since this is true for all $X_{M} \in \mathbb{R}$ and $g \neq 0$, we conclude that the distribution is not strongly exponential, which contradicts our hypothesis.

Then, we prove $A$ is invertible. Collect points into $k$ vectors $(\overline{X}_{M}^{1}, ..., \overline{X}_{M}^{k})$, and concatenate the $k$ Jacobian $J_{\mathbf{H}}(\overline{X}_{M}^{l})$ evaluated at each of those vectors horizontally into matrix $Q=(J_{\mathbf{H}}(\overline{X}_{M}^{1}), ..., J_{\mathbf{H}}(\overline{X}_{M}^{k}))$ and similarly define $\hat{Q}$ as the concatenation of the Jacobian of $\hat{\mathbf{H}}(\hat{f}^{-1} \circ f(\overline{X}_{M}))$ evaluated at those points. Then the matrix $Q$ is invertible. By differentiating Eq.~\ref{H} for each $X_{M}^{l}$, we have:
\begin{equation}
    Q=A\hat{Q}
    \label{QAQ}
\end{equation}
The invertibility of $Q$ implies the invertibility of $A$ and $\hat{Q}$, which completes the proof.
\end{proof}

\textcolor{black}{Theorem~\ref{identifiability theorem} is a basic form of identifiability of generative model Eq.~\ref{eq12}. The latent variables $\hat{\mathbf{M}}$ can be recovered by a permutation (the matrix $A$) and a point-wise nonlinearity (in the form of $H$ and $\hat{H}$) of the original latent variables $\mathbf{M}$.} 

\section{Experiments}
In this section, we first conduct experiments on synthetic and semi-synthetic datasets to validate that the proposed CPTiVAE can efficiently estimate causal effects with the latent confounders and post-treatment variables. We use the causal DAG in Figure~\ref{CMiVAE SCM} to generate synthetic datasets for evaluating the performance of CPTiVAE. We also perform a sensitivity analysis on the model parameters and verify the feasibility of the algorithm when the learned representation dimensions are different from the true covariate dimensions. Finally, we apply our CPTiVAE algorithm on a real-world dataset to demonstrate its potential application. The code is available in the appendix.

\begin{table*}[t]
\setlength\tabcolsep{3pt}
\centering
\begin{tabular}{ccccccccccccc}
\toprule
\multirow{2}{*}{Method} &  & \multicolumn{5}{c}{ATE}      &  & \multicolumn{5}{c}{CATE}     \\ \cmidrule{3-7} \cmidrule{9-13} 
                        &  & 2K & 4K & 6K & 8K & 10K &  & 2K & 4K & 6K & 8K & 10K \\ \midrule
LDML                    &  & 3.01 ± 0.01 & 3.00 ± 0.01  & 3.00 ± 0.01 & 2.99 ± 0.01 & 3.00 ± 0.01 &  & 3.00 ± 0.01 & 3.01 ± 0.01  & 3.00 ± 0.01 & 3.00 ± 0.01 & 3.00 ± 0.01 \\
LDRL                    &  & 3.00 ± 0.01 & 3.01 ± 0.01  & 2.99 ± 0.01 & 3.01 ± 0.01 & 3.00 ± 0.01 &  & 2.99 ± 0.01 & 3.00 ± 0.01  & 3.00 ± 0.01 & 2.99 ± 0.00 & 3.00 ± 0.01 \\
KDML                    &  & 2.97 ± 0.01 & 2.98 ± 0.01  & 2.96 ± 0.01 & 2.97 ± 0.01 & 2.98 ± 0.01 &  & 2.98 ± 0.01 & 2.98 ± 0.01  & 2.97 ± 0.01 & 2.98 ± 0.01 & 2.99 ± 0.01 \\
CFDML                   &  & 3.00 ± 0.01 & 3.02 ± 0.01  & 3.01 ± 0.01 & 3.02 ± 0.01 & 3.01 ± 0.01 &  & 3.01 ± 0.01 & 3.01 ± 0.01  & 3.02 ± 0.01 & 3.00 ± 0.01 & 2.99 ± 0.01 \\
SLDML                   &  & 2.98 ± 0.01 & 2.99 ± 0.01  & 2.98 ± 0.01 & 2.97 ± 0.01 & 2.99 ± 0.01 &  & 3.00 ± 0.01 & 2.99 ± 0.01  & 3.00 ± 0.01 & 2.99 ± 0.01 & 2.98 ± 0.01 \\
GANITE                  &  & 2.45 ± 0.13 & 2.44 ± 0.01 & 2.46 ± 0.01 & 2.46 ± 0.01 & 2.45 ± 0.01 &  & 3.66 ± 0.01 & 3.67 ± 0.01 & 3.67 ± 0.01 & 3.68 ± 0.01 & 3.67 ± 0.01 \\
X-learner               &  & 3.01 ± 0.01 & 3.00 ± 0.01 & 2.99 ± 0.01 & 3.00 ± 0.01 & 3.00 ± 0.01 &  & 3.00 ± 0.01 & 3.01 ± 0.01 & 2.99 ± 0.01 & 3.01 ± 0.01 & 3.00 ± 0.01 \\
R-learner               &  & 3.02 ± 0.01 & 3.00 ± 0.01 & 3.01 ± 0.01 & 2.99 ± 0.01 & 3.00 ± 0.01 &  & 3.07 ± 0.01 & 3.02 ± 0.01 & 3.01 ± 0.01 & 3.01 ± 0.01 & 3.01 ± 0.01 \\
CEVAE                   &  & 2.93 ± 0.26 & 2.98 ± 0.37 & 1.15 ± 0.79 & 1.89 ± 0.47 & 1.59 ± 0.64 &  & 2.60 ± 0.65 & 2.95 ± 0.91 & 3.20 ± 0.95 & 3.54 ± 0.30 & 3.33 ± 0.66 \\
TEDVAE                  &  & 3.69 ± 0.02 & 3.59 ± 0.29 & 2.92 ± 0.84 & 2.87 ± 1.31 & 2.61 ± 0.76 &  & 3.69 ± 0.02 & 3.62 ± 0.24 & 3.21 ± 0.36 & 3.11 ± 0.67 & 2.78 ± 0.47 \\
CPTiVAE                 &  & \multicolumn{1}{c}{\textbf{0.90 ± 0.05}}  & \multicolumn{1}{c}{\textbf{0.84 ± 0.08}} & \multicolumn{1}{c}{\textbf{0.83 ± 0.03}} & \multicolumn{1}{c}{\textbf{0.81 ± 0.03}} & \multicolumn{1}{c}{\textbf{0.70 ± 0.09}} &  & \multicolumn{1}{c}{\textbf{0.92 ± 0.04}}  & \multicolumn{1}{c}{\textbf{0.87 ± 0.05}} & \multicolumn{1}{c}{\textbf{0.84 ± 0.03}} & \multicolumn{1}{c}{\textbf{0.87 ± 0.06}} & \multicolumn{1}{c}{\textbf{0.82 ± 0.04}} \\ \bottomrule
\end{tabular}
\caption{The performance for estimating ATE and CATE on synthetic datasets with different sample sizes. The results are reported by AE and PEHE (mean ± standard deviation), respectively. The best result is marked in boldface.}
\label{abs error}
\end{table*}

\subsection{Experiment Setup}
\paragraph{Baseline methods} 
We compare CPTiVAE with ten state-of-the-art causal effect estimators that are widely used to estimate ATE and CATE from observational data. LinearDML (LDML)~\cite{chernozhukov2018double}, LinearDRLearner (LDRL)~\cite{foster2023orthogonal}, KernelDML (KDML)~\cite{nie2021quasi}, CausalForestDML (CFDML)~\cite{Athey2019}, SparseLinearDML (SLDML)~\cite{semenova2017estimation}, GANITE~\cite{yoon2018ganite}, and Mete-learners (including X-learner and R-learner)~\cite{kunzel2019metalearners} are the machine learning based estimators. There are two VAE-based estimators, CEVAE~\cite{louizos2017causal} and TEDVAE~\cite{zhang2021treatment}.

\paragraph{Evalution metrics}
We report the Absolute Error (AE) to evaluate ATE estimation, where $\text{AE}=|\hat{\text{ATE}}-\text{ATE}|$, ATE is the true causal effect and $\hat{\text{ATE}}$ is the estimated causal effect.

We use the Precision of the Estimation of Heterogeneous Effect (PEHE) to evaluate CATE estimation, defined as,

\begin{equation}
    {\varepsilon_{\text{PEHE}}}=\sqrt{\frac{1}{N} \sum_{i=1}^{N} (\hat{\text{ITE}}-\text{ITE})^2},
\end{equation}
where $\hat{\text{ITE}}$ is the estimated individual treatment effect and $\text{ITE}$ is the true individual treatment effect.

\paragraph{Implementation details}
CPTiVAE is mainly implemented by \textit{Python} and includes the libraries \textit{pytorch}~\cite{paszke2019pytorch} and \textit{pyro}~\cite{bingham2019pyro}.

The implementations of LDML, LDRL, KDML, CFDML, and SLDML are from the \textit{Python} package \textit{econml}. The implementation of GANITE is from the authors' GitHub\footnote{\url{https://github.com/vanderschaarlab/mlforhealthlabpub/tree/main/alg/ganite}}. The implementations of X-learner and R-learner are from the \textit{Python} package \textit{CausalML}~\cite{Zhao2020causalml}. We use \textit{Python} library \textit{pyro} to achieve CEVAE and the implementation of TEDVAE is from the authors' GitHub\footnote{\url{https://github.com/WeijiaZhang/TEDVAE}}.

\begin{table}[t]
\centering
\begin{tabular}{cccc}
\toprule[1pt]
\multirow{2}{*}{Weight} & \multicolumn{3}{c}{Sample Sizes} \\ \cmidrule{2-4}                                                                          
& \multicolumn{1}{c}{6K} & \multicolumn{1}{c}{8K} & \multicolumn{1}{c}{10K} \\
\midrule
$\left \{\alpha, \beta \right \}=0.01$ & 0.89 ± 0.10 & 0.82 ± 0.40 & 0.92 ± 0.06 \\
$\left \{\alpha, \beta \right \}=0.1$ & 0.99 ± 0.08 & 0.91 ± 0.07 & 0.99 ± 0.22 \\
$\left \{\alpha, \beta \right \}=1$ & 0.83 ± 0.03 & 0.81 ± 0.13 & 0.70 ± 0.09 \\
$\left \{\alpha, \beta \right \}=10$ & 1.16 ± 0.10 & 0.84 ± 0.10 & 1.09 ± 0.05 \\
$\left \{\alpha, \beta \right \}=100$ & 1.08 ± 0.04 & 0.92 ± 0.03 & 1.01 ± 0.02 \\
\bottomrule[1pt]
\end{tabular}
\caption{The AE (mean ± standard deviation) with the different setting of tuning parameters $\alpha$ and $\beta$.}
\label{parameter analysis}
\end{table}

\subsection{Evalution on Synthetic Datasets}
We use the causal DAG in Figure~\ref{CMiVAE SCM} to generate the synthetic datasets with sample sizes, 2K, 4K, 6K, 8K, and 10K for our experiments. 

In the causal DAG $\mathcal{G}$, $\mathbf{C}$ and $\mathbf{M}$ are latent variables, $\mathbf{X_{C}}$ and $\mathbf{X_{M}}$ are the proxy variables. $\mathbf{C}$ is generated from Bernoulli distribution. For a post-treatment variable $M \in \mathbf{M}$, it is generated from the treatment variable $T$ by using $M=\eta_{1} T+ \varepsilon_{M}$, where $\eta_{1}$ is coefficient, $\varepsilon_{M}$ is the noise term. $X_{C}$ and $X_{M}$ are generated from the latent confounder $C$ and the post-treatment variable $M$, defined as $X_{C} \sim N(C,\eta_{2}C)$ and $X_{M} \sim N(M, \eta_{3}M)$, where $\eta_{2}$ and $\eta_{3}$ are two coefficients. We use the Bernoulli distribution with the conditional probability to generate the treatment $T$, defined as $P(T=1|\mathbf{C})=0.75C+0.25(1-C)$. Then, we can define $Y(T=t)=T+3C+M+\varepsilon_{Y}$, where $\varepsilon_{Y}$ is an error term. 

We can obtain the true ITE for an individual based on the data generation process, and the true ATE and CATE are 1. We repeated the experiment 30 times independently for each setting to evaluate the performance of CPTiVAE.

\begin{figure}[t]
    \centering
    \subfigure[ATE]{
    \includegraphics[width=0.22\textwidth]{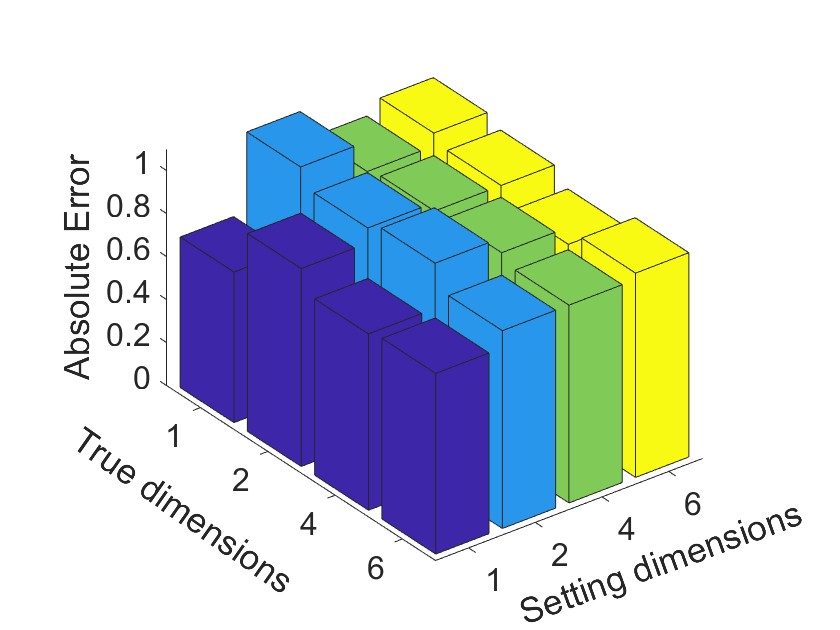}
    \label{Abs Error}
    }
    \subfigure[CATE]{
    \includegraphics[width=0.22\textwidth]{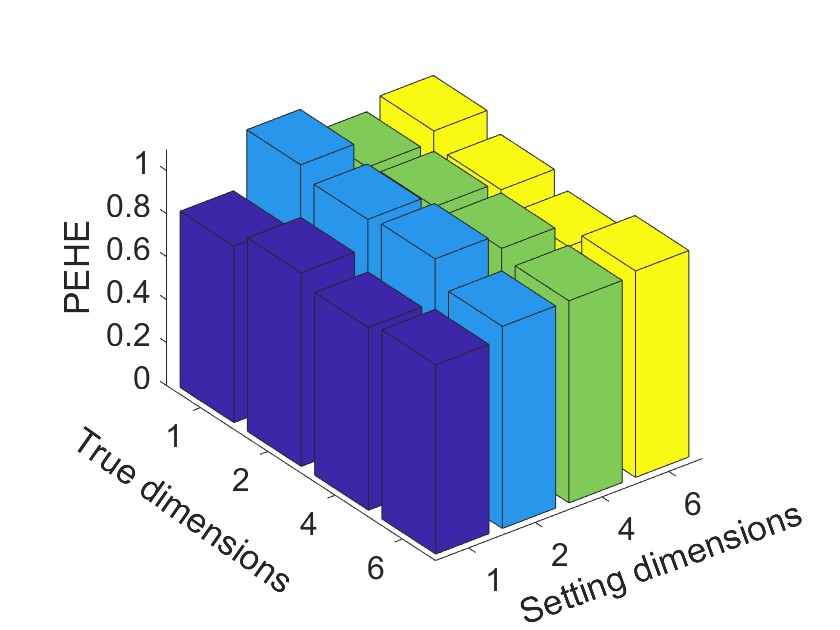}
    \label{PEHE_10k}
    }
    \caption{The results of the dimensionality study. ``True dimensions'' refer to the dimensions of $\mathbf{C}$ and $\mathbf{M}$ in the data, and ``Setting dimensions'' correspond to the parameters of $D_{\mathbf{C}}$ and $D_{\mathbf{M}}$ in the CPTiVAE algorithm.}
    \label{dimension study}
\end{figure}

\begin{figure}[t]
    \centering
    \subfigure[ATE]{
    \includegraphics[width=0.47\textwidth]{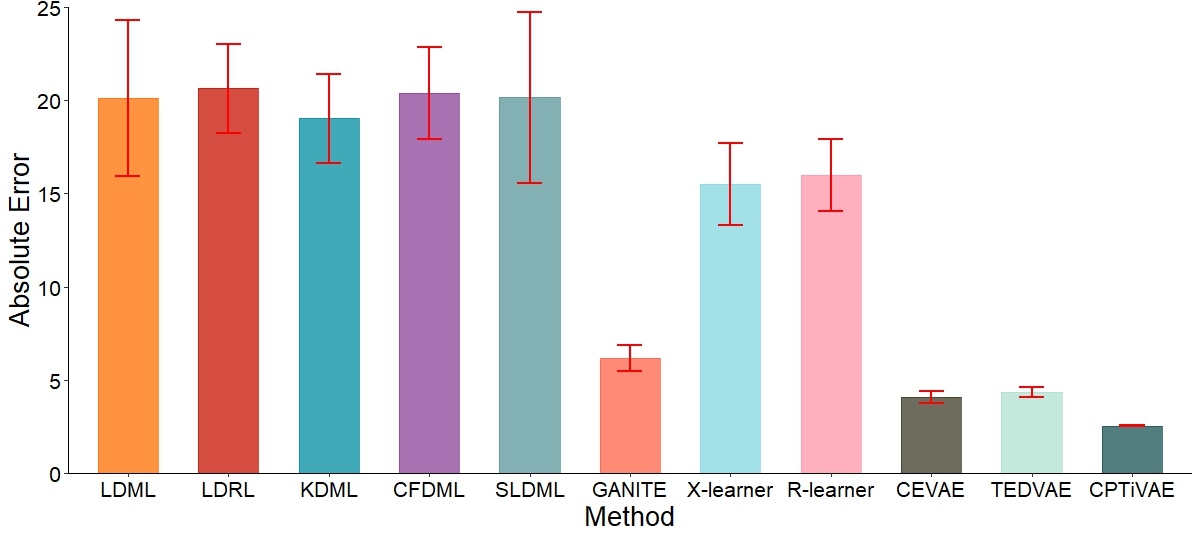}
    \label{IHDP_ATE}
    }
    \subfigure[CATE]{
    \includegraphics[width=0.47\textwidth]{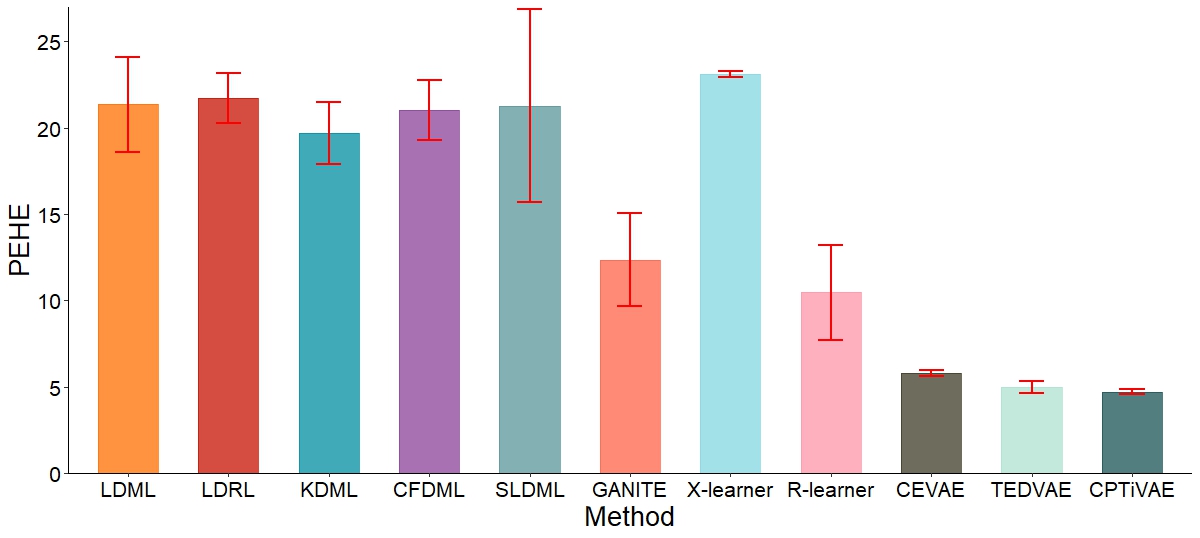}
    \label{IHDP_PEHE}
    }
    \caption{The performance for estimating ATE and CATE on the semi-synthetic dataset.}
    \label{dimension study semi}
\end{figure}

We report the AE and PEHE for synthetic datasets in Tables~\ref{abs error}. From the experimental results, we can know that the estimators based on machine learning, i.e., LDML, LDRL, KDML, CFDML, SLDML, GANITE, X-learner and R-learner have a large AE and PEHE on synthetic datasets since these estimators cannot learn the representation from proxy variables. TEDVAE and CEVAE have a large AE and PEHE on synthetic datasets since both methods cannot tackle the latent post-treatment variables. Our CPTiVAE algorithm obtains the smallest AE and PEHE among all methods on both synthetic datasets since our CPTiVAE algorithm can learn the representation of the latent post-treatment variable from proxy variables.

In summary, the simulation experiments show that the CPTiVAE algorithm proposed in this paper can effectively address the bias introduced by the latent post-treatment variables when estimating ATE and CATE from observational data. Moreover, it further shows that CPTiVAE can recover latent variable representations from proxy variables.

\subsection{Parameter Analysis}
To balance $\mathcal{L}_{\text{ELBO}}$ and the two classifiers during the training process, we add two tuning parameters $\alpha$ and $\beta$ to our CPTiVAE algorithm. We use the synthetic datasets with a sample size of 6K, 8K, and 10K, generated by the same data generation process described in Section 5.2, to analyze the sensitivity of the two parameters. 

In this paper, we set $\left \{\alpha, \beta \right \}= \left \{0.01, 0.1, 1, 10, 100\right \}$ and the results of the experiments show in Table~\ref{parameter analysis}. From Table~\ref{parameter analysis}, the two parameters $\alpha$ and $\beta$ are less sensitive to the AE of the CPTiVAE algorithm in ATE estimation. In summary, we set the two parameters $\alpha$ and $\beta$ to small values for our CPTiVAE algorithm.

\subsection{Dimensionality Study}
There are two representations $\mathbf{C}$ and $\mathbf{M}$ in our study and we set their dimensions to 1 for both. To demonstrate the effectiveness of this setting, we conduct experiments on datasets with samples fixed at 10k and repeat the experiments independently 30 times to minimize random noise for each setting. We generate a set of simulation datasets based on the data generation process described in Section 5.2, setting the dimensions of the two latent variables $(\mathbf{C}, \mathbf{M})$ to $\left \{1, 2, 4, 6 \right \}$, respectively. In our CPTiVAE algorithm, we conduct experiments with the two parameters $(D_{\mathbf{C}}, D_{\mathbf{M}})$ set to $\left \{1, 2, 4, 6 \right \}$ respectively. 

Figure~\ref{dimension study} shows the AE and PEHE of the CPTiVAE algorithm on these datasets. When $(D_{\mathbf{C}}, D_{\mathbf{M}})$ is set to $(1,1)$ regardless of the true dimensions of $\mathbf{C}$ and $\mathbf{M}$, the CPTiVAE algorithm obtains the smallest AE and PEHE on the synthetic datasets. Furthermore, our experiments demonstrate that the CPTiVAE algorithm is effective in addressing high-dimensional latent variables. Hence, it is reasonable to set $D_{\mathbf{C}}$ and $D_{\mathbf{M}}$ to 1.

\subsection{Experiments on Semi-Synthetic Datasets}
In this section, we compare our CPTiVAE algorithm with the baseline methods on the semi-synthetic dataset based on the IHDP dataset. The Infant Health and Development Program (IHDP) is a randomized controlled study designed to evaluate the effect of physician home visits on infants' cognitive scores. The IHDP dataset is often used as a baseline dataset to evaluate causal effect methods~\cite{hill2011bayesian}. The dataset contains 747 samples and 25 variables, and in our experiments, following the synthetic dataset generation process described in Section 5.2, two post-treatment variables were generated based on these 25 variables. 

The performance of the CPTiVAE algorithm on semi-synthetic datasets is shown in Figure~\ref{dimension study semi}. We observe that the CPTiVAE algorithm achieves the smallest AE and PEHE. This further illustrates that the CPTiVAE algorithm can effectively tackle both latent confounders and post-treatment variables.

\begin{figure}[t]
    \centering
    \includegraphics[width=0.5\textwidth]{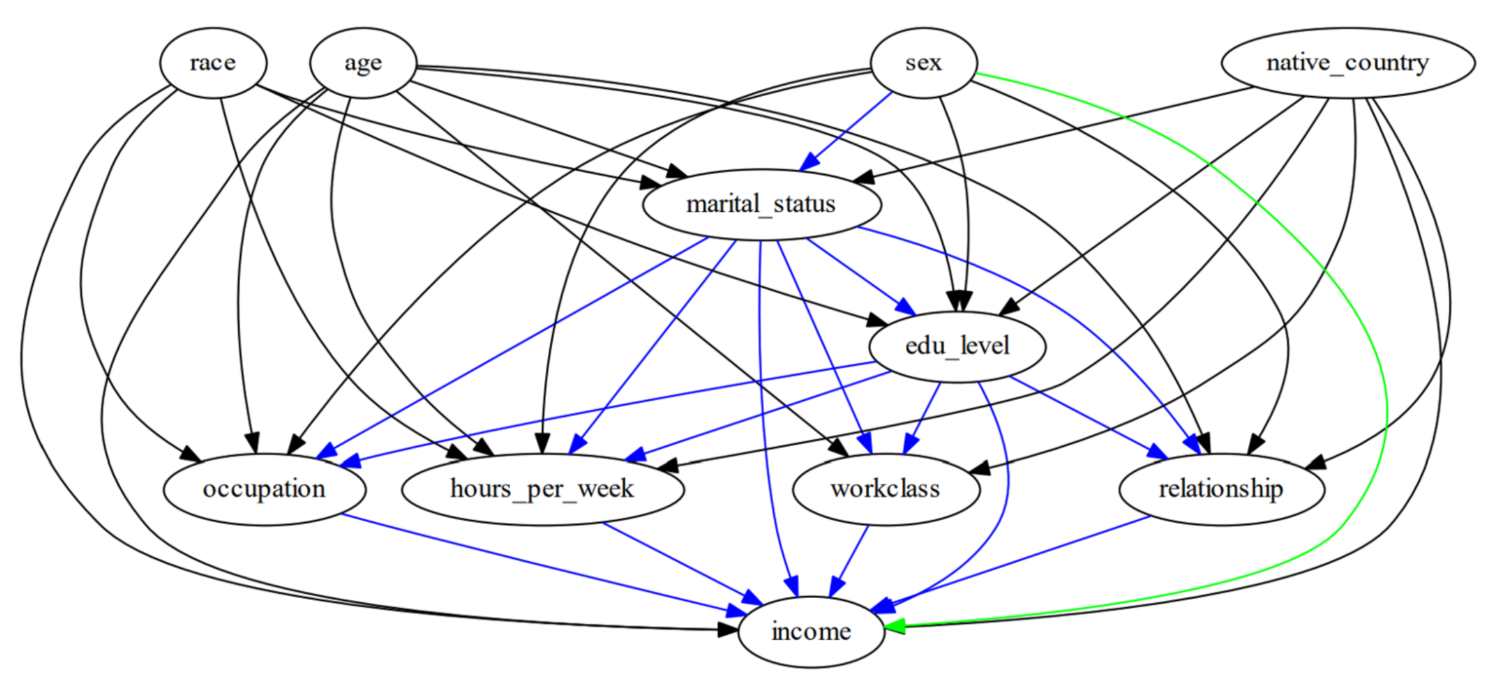}
    \caption{The causal network for the Adult dataset: the green path represents the direct path, and the blue paths represent the indirect paths passing through marital status \protect \cite{ijcai2017p549}.}
    \label{Adult_all}
\end{figure}

\subsection{Case Study on a Real-World Dataset}
In this section, we apply the CPTiVAE algorithm to study the causal effect of education level on income in the Adult dataset~\cite{causal2017Li}. The Adult dataset contains 14 attributes for 48,842 individuals, which is retrieved from the UCI repositiry~\cite{dua2017uci}. In our problem setting, some extraneous variables, such as sex, relationship, and race, are not relevant to our study, and we remove these variables. For convenience, we divide the years of education into 2 groups, those greater than 10 years as the treatment group and the rest as the control group. We define education-num as $T$, income as $Y$, age and marital status as $\mathbf{X_{C}}$, and the remaining attributes as $\mathbf{X_{M}}$, including workclass, occupation, and hours-per-week.

\begin{figure}
    \centering
    \includegraphics[width=0.4\textwidth]{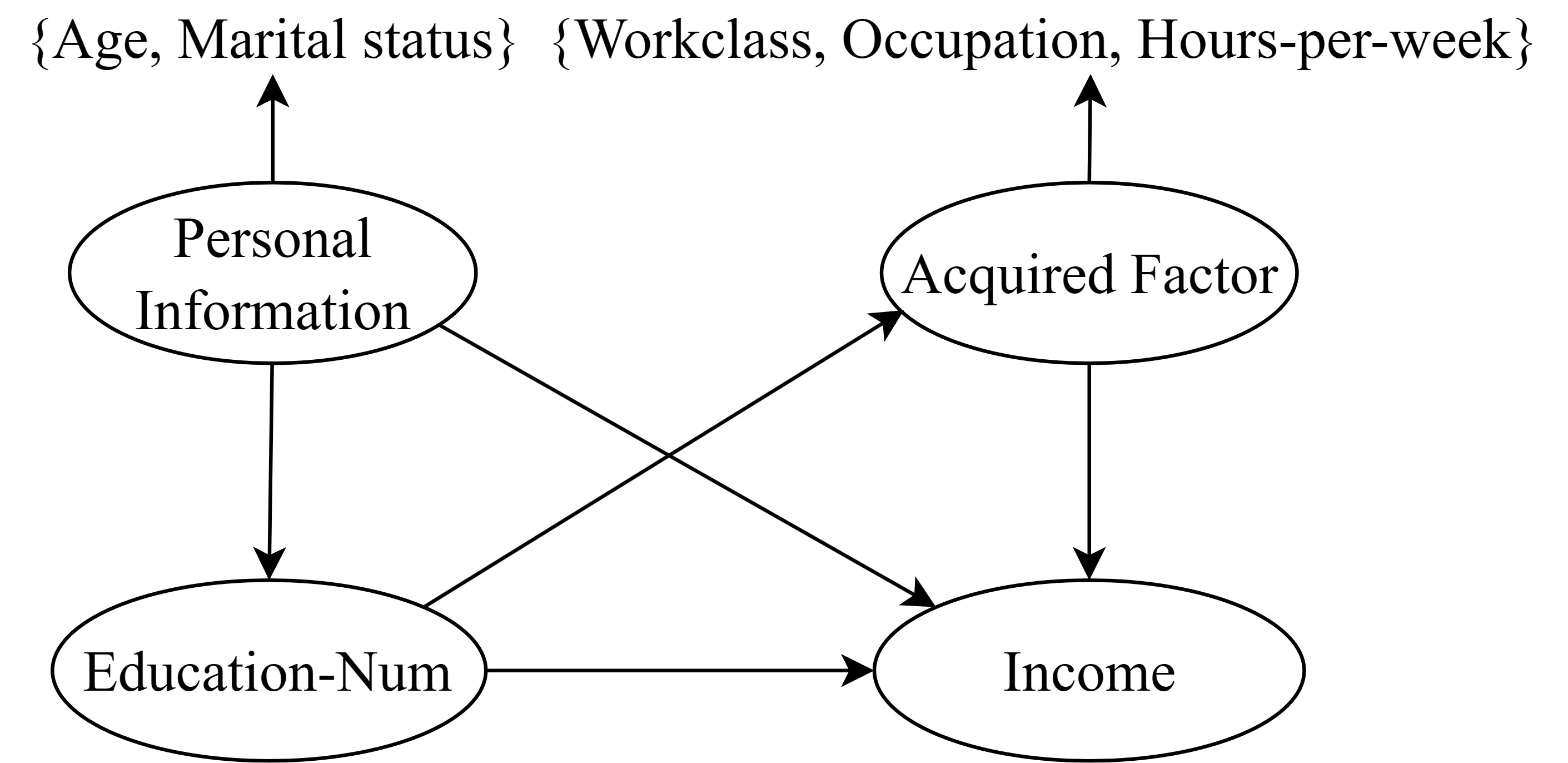}
    \caption{Simplified DAG for Adult dataset.}
    \label{Adult DAG}
\end{figure}


\textcolor{black}{The causal network for the Adult dataset is shown in~\ref{Adult_all}.} Based on the complete causal DAG, we can obtain the simplified DAG in Figure~\ref{Adult DAG}, and aim to estimate the causal effect of education level on income. In this example, we consider the proxy of acquired factors to be accessible, and they are workclass, occupation, and hours-per-week. 

We obtain $ATE=9.245$ by using the CPTiVAE algorithm, which is consistent with the conclusion that education has a positive effect on income~\cite{Hu_Wu_Zhang_Wu_2021}. In contrast, we obtain $ATE=0.053$ if all covariates are defined as $\mathbf{X_{C}}$, and $ATE=0.180$ if only considering the latent confounders. Both indicate that education level has no or marginal effects on income which is not the correct conclusion. This shows that if we do not consider $\mathbf{X_{M}}$, the causal effect estimation will be biased.

\section{Conclusion}
In this paper, we address the challenging problem of estimating causal effects from observational data with latent confounders and post-treatment variables. We propose the use of the VAE and the identifiable VAE techniques for learning the representations of latent confounders and post-treatment variables, respectively. Furthermore, we provide theoretical proof for the identifiability in terms of the learned representation $\mathbf{M}$. Extensive experiments on synthetic and semi-synthetic datasets demonstrate that CPTiVAE outperforms existing methods for ATE and CATE estimation. We also examine that CPTiVAE's performance remains robust against variations in the model parameters $\alpha$ and $\beta$. Additionally, a case study on the real-world dataset further illustrates CPTiVAE's potential in practical scenarios.

\newpage
\bibliographystyle{ieeetr}
\bibliography{CPTiVAE.bib}
\end{document}